\documentclass[11pt]{article}
\usepackage{booktabs}
\usepackage{graphicx}
\usepackage{caption}
\usepackage{amssymb,amsmath,amsthm}
\usepackage{xcolor}
\usepackage{thmtools}
\usepackage{amsfonts}
\usepackage{hyperref}
\usepackage{enumerate}
\usepackage[margin=1in]{geometry} 

\newcommand{\cdf}[1]{{\color{red} #1 }}

\DeclareMathOperator{\Tr}{Tr}

\newtheorem{theorem}{Theorem}

\usepackage[american]{babel}
\usepackage{mathtools} % amsmath with fixes and additions
\usepackage{booktabs} % commands to create good-looking tables
\usepackage{tikz} % nice language for creating drawings and diagrams

\title{AUTM Flow: Atomic Unrestricted Time Machine \\
for Monotonic Normalizing Flows}

\author{Difeng Cai\thanks{Department of Mathematics, Emory University, USA} \and Yuliang Ji$^{\ast}$ \and Huan He\thanks{Department of Computer Science, Emory University, USA} \and Qiang Ye\thanks{Department of Mathematics, University of Kentucky, USA} \and Yuanzhe Xi$^{\ast}$}

\begin{document}
\maketitle

\begin{abstract}
  Nonlinear monotone transformations are used extensively in normalizing flows to construct invertible triangular mappings from simple distributions to complex ones.
In existing literature, monotonicity is usually enforced by restricting function classes or model parameters and the inverse transformation is often approximated by root-finding algorithms as a closed-form inverse is unavailable.
In this paper, we introduce a new integral-based approach termed "Atomic Unrestricted Time Machine (AUTM)", equipped with unrestricted integrands and easy-to-compute explicit inverse.
%where the inverse is obtained by simply swapping the integral limits in the forward transformation.
AUTM offers a versatile and efficient way to the design of normalizing flows with explicit inverse and unrestricted function classes or parameters.
% Unlike neural ODEs that rely on the expressive power of neural networks, the AUTM coupling and autoregressive flows allow unrestricted function classes to model the dynamics which can be as simple as polynomials, and the input dimensions are fully decoupled, giving rise to a much smaller Lipschitz constant and a lower triangular Jacobian.
Theoretically, we present a constructive proof that AUTM is universal: all monotonic normalizing flows can be viewed as limits of AUTM flows.
We provide a concrete example to show how to approximate any given monotonic normalizing flow using AUTM flows with guaranteed convergence.
The result implies that AUTM can be used to transform an existing flow into a new one equipped with explicit inverse and unrestricted parameters. The performance of the new approach is evaluated on high dimensional density estimation, variational inference and image generation.
Experiments demonstrate superior speed and memory efficiency of AUTM.
\end{abstract}

\section{Introduction}

Generative models aim to learn a latent distribution from given samples and then generate new data from the learned distribution. There are several kinds of generative models, including
generative adversarial networks (GANs) \cite{GAN_goodfellow}, variational autoencoders  (VAE) \cite{kingma2014autoencoding}, and normalizing flows \cite{flow15danilo}, etc. 
Unlike GANs and VAE, normalizing flows offer a tractable and efficient way for exact density estimation and sampling. 
Applications of normalizing flows include image generation \cite{GLOW18,flow++2019}, noise modelling \cite{noiseflow2019}, and reinforcement learning \cite{offpolicy2019}, et al.

% In practice, a central question in normalizing flow is:
% \begin{quote}
% how to design a class of monotone triangular mappings that are expressive enough to model complex distributions and in the mean time admit efficient computation of (hopefully \emph{explicit}) inverse and Jacobian determinant ?
% \end{quote}

%Normalizing flows learn an invertible transformation to map the simple base distribution to the complex target distribution such that
Two challenges in normalizing flows are the computation of Jacobian determinant and the inverse transformation. 
Different architectures have been proposed to address those issues. 
Neural ODE \cite{NeuralODE18} and Free-form Jacobian of Reversible Dynamics (FFJORD) \cite{FFJORD} pioneered the way of modeling the transformation as a dynamical system.
The inverse can be easily computed by reversing the dynamics in time, but Jacobian determinant is hard to compute and the use of neural network to model the dynamics often leads to high computational cost.
To simplify the Jacobian computation, most flows employ monotone triangular mappings so that the Jacobian is triangular. Two such architectures are autoregressive flows and coupling flows.
% This is because it has been shown that monotone triangular mappings are able to transform any base distribution to a target distribution \cite{triangular2005}. \textcolor{blue}{YX: briefly mention neural ode approach here.}
% Autoregressive flows 
% %(cf. \cite{MAF17,naf2018,bnaf20,UMNN}) 
% and coupling flows 
% %(cf. \cite{nice2014,realNVP17,GLOW18,neuralspline19,flow++2019}) 
% are the two most widely used flow architectures based on stacking triangular invertible transformations.
% % Both architectures employ strictly monotone transformation and the Jacobian of the overall transformation is triangular.
Examples of monotone mappings used in those flows include 1) \textbf{Special function classes} such as affine function \cite{nice2014,realNVP17,GLOW18}),
   rational function \cite{ziegler2019},     logistic mixture \cite{flow++2019},
    splines \cite{cubicspline,neuralspline19}; 2) \textbf{Neural networks} \cite{naf2018,bnaf20};
    and 3) \textbf{Integral of positive functions} \cite{sumsquare19,UMNN}.
%\begin{itemize}
%    \item \textbf{Special function classes}: affine function (\cite{nice2014,realNVP17,GLOW18}),
%    rational function (\cite{ziegler2019}), 
%    logistic mixture (\cite{flow++2019}),
%    splines (\cite{cubicspline,neuralspline19}).
%    \item \textbf{Neural networks}: \cite{naf2018,bnaf20}.
%    \item \textbf{Integral of positive functions}: \cite{sumsquare19,UMNN}.
%\end{itemize}

To ensure monotonicity, methods using ``special function classes" such as splines or sigmoid function $\sigma$ (see Table \ref{tab:CF}) have to impose constraints on model parameters, which often impede the expressive power of the transformation as well as the efficiency of training.
For example, during optimization, updates like  $\theta=\theta-\gamma\nabla_{\theta}\mathcal{L}$ can potentially make the parameter $\theta$ fall out of the prescribed range and modifications are needed to guarantee the monotonicity under the new update.
Integral-based methods rely on the simple fact that the function
% \begin{equation}
% \label{eq:umnn}
    $q(x) = c + \int_0^x g(x) dx$
% \end{equation}
is always increasing as long as the integrand $g$ is globally \emph{positive}.
For example, $g$ is modeled as positive polynomials in sum-of-squares (SOS) polynomial ﬂow by \cite{sumsquare19}
and as positive neural networks in unconstrained monotonic neural networks (UMNN) by \cite{UMNN}.
Due to the flexibility offered by an integral form, these methods allow \emph{unrestricted} model parameters as compared to other methods.
It was shown in \cite{UMNN} that the method requires fewer parameters than straightforward neural network-based methods in \cite{naf2018,bnaf20} and can scale to high dimensional datasets.
However, unlike Neural ODE and FFJORD, these integral-based monotone mappings do not possess an explicit inverse formula and one has to resort to root-finding algorithms to compute the inverse transformation.
This leads to increased computational cost because in \emph{each} iteration of root-finding, one has to compute an integral of a complicated function.
% (with a possibly large Lipschitz constant)
As discussed in \cite{UMNN}, a judicious choice of quadrature rule is needed.
%in order to maintain the accuracy and efficiency. 

\paragraph{Contributions}
In this paper, we propose a new integral-based monotone triangular flow with proven universal approximation property for monotonic flows.
% Hence the model does \emph{not} require the integrand to be positive.
% The new transformation admits an \emph{explicit} inverse formula and \emph{unrestricted} model classes and parameters. 
% We prove in Section \ref{sec:theory} the universality of the proposed flow.
The major contributions include the following.
\begin{itemize}
    \item \textbf{Unrestricted model classes and parameters.} The proposed integral-based transformation is strictly increasing with \emph{no} constraint on the integrand except being Lipschitz continuous.
    \item \textbf{Explicit inverse.} The inverse formula is explicitly given and compatible with fast root-finding methods for  numerical inversion.
    \item \textbf{Universality.} The proposed model is \emph{universal} in the sense that \emph{any} monotonic normalizing flow
    %, for example, the affine coupling flows (cf. \cite{nice2014,realNVP17}),
    is a limit of the proposed flows.
\end{itemize}

\section{Background}
Normalizing flows are a class of generative models that aim to find a bijective mapping $f$ such that $f^{-1}$ ``normalizes" the complex distribution into a tractable base distribution (Gaussian, for example).
Once the ``normalizing" mapping is found, generating new data points boils down to simply sampling from the base distribution and applying the forward transformation $f$ to the samples.
This makes normalizing flows a popular choice in density estimation, variational inference, image generation, etc.

Let $Y\in \mathbb{R}^D$ be a random variable with a possibly complicated probability density function $p_Y(y)$
and $X \in \mathbb{R}^D$ be a random variable with a well-studied probability density function $p_X(x)$. 
Assume that there is an invertible (vector-valued) function $f:\mathbb{R}^D\to\mathbb{R}^D$ that transforms the ``base" variable $X$ to $Y$, i.e. $Y=f(X)$. 
Then according to the change of variables formula, the probability density functions $p_Y$ and $p_X$ satisfy
\begin{equation}
\label{formula_of_flow}
    p_Y(y)=p_X(x)\left| \det J_{f^{-1}}(y) \right|=p_X(x) \left| \det J_{f}(x) \right|^{-1},
\end{equation}
where $J_f(x)$ denotes the Jacobian of $f$ evaluated at $x$.

\iffalse
\cdf{Needed?:}
For multi-layer architectures like coupling flows and autoregressive flows, the mapping $f$ is a composition of multiple bijective mappings.
If $f_1,...,f_N$ are a set of $N$ invertible functions on $\mathbb{R}^D$ and $f= f_N \circ f_{N-1} \circ ... \circ f_1$ is the composition of $f_i$'s, then the formula in \eqref{formula_of_flow} becomes
\begin{equation}
\label{formula_of_flow2}
    p_X(x)=p_Z(z)\prod_{i=1}^N|\det J_{f_i}^{-1}(x_i)|
    =p_Z(z)\prod_{i=1}^N|\det J_{f_i}(z_i)|^{-1},
\end{equation}
where $x_i=f_{i+1}^{-1} \circ ... \circ f_N^{-1}(x)$ ($i<N$), $x_N = x$,
$z_i=f_{i-1}\circ\dots\circ f_1(z)$($i>1$), and $z_1=z$.
The log-density of $x$ then satisfies
\begin{equation}
\label{formula_of_flow3}
    \log p_X(x)= \log p_Z(z)+\sum_{i=1}^N \log |\det _i^{-1}(x_i)|
    =  \log p_Z(z)-\sum_{i=1}^N \log |\det J_{f_i(z_i)|.
\end{equation}
\fi

The success of normalizing flows hinges on many factors, including \textbf{1:} the expressive power of the class of bijective transformations;
\textbf{2:} the efficiency in computing the transformation and its inverse;
\textbf{3:} the efficiency in computing the Jacobian determinant.
% \end{itemize}
% Since the expressiveness of $f$ affects the difficulty in computing the inverse and Jacobian, it is extremely challenging to design a good architecture that magnifies the expressiveness without losing computational efficiency.

A popular architectural design to address those points is to employ an invertible triangular transformation, whose Jacobian is triangular and inversion can be computed in an entrywise fashion.
Two representative triangular normalizing flows are autoregressive flows and coupling flows.
% The invertibility is usually guaranteed by choosing monotone mappings. 
% See Table \ref{tab:CF} for some examples.
Throughout the paper, we use $x_{1:k}$ to denote the vector $(x_1,\dots, x_k)$.

\subsection{Autoregressive flows}
Autoregressive flows choose $f:\mathbb{R}^D\to \mathbb{R}^D$ to be an autoregressive mapping:
\begin{equation}
\label{eq:AF}
\begin{split}
        f(x;&\theta) =  (q_1(x_1;\theta_1),q_2(x_2;\theta_2(x_1)),\dots, \\
        &q_k(x_k;\theta_k(x_{1:k-1})),
        \dots,q_D(x_D;\theta_D(x_{1:D-1}))),
\end{split}
\end{equation}
where $x=(x_1,\dots,x_D)$ and each $q_k$, termed \emph{transformer}, is a bijection parametrized by the so-called \emph{conditioner} $\theta_k(x_{1:k-1})$.
The Jacobian of $f$ is a lower triangular matrix.
To guarantee the invertibility of $f$, $q_k$ is chosen to be a monotone function of $x_k$.
Since $q_k$ is usually a nonlinear neural network, an analytic inverse is not available and the inverse is computed by root-finding algorithms.
% Different from ODE-based models, since the determinant of Jacobian is easy to compute,
% autoregressive flows usually  stack multiple bijective mappings in order to increase the expressive power.

\subsection{Coupling flows}
Coupling flows 
%\cite{nice2014,realNVP17} 
% were proposed to facilitate the computation of Jacobian determinant of the transformation and the bijective mapping.
first partition the input vector $x=(x_1,\dots,x_D)$ into two parts $x_{1:d}$ and $x_{d+1:D}$ and then apply the following transformation:
\begin{equation}
\label{eq:CF}
\begin{aligned}
    y_{1:d} &= x_{1:d},\quad
    y_{d+1:D} &= q(x_{d+1:D}; \theta(x_{1:d})),
\end{aligned}
\end{equation}
where the parameter $\theta(x_{1:d})$ is an arbitrary function of $x_{1:d}$ and the scalar \emph{coupling function} $q$ is applied entrywisely, i.e., $q(x_{d+1:D}) = ( q(x_{d+1}),\dots, q(x_{D}) )$.
The transformation in \eqref{eq:CF} from $x$ to $y$ is called a \emph{coupling layer}.
In normalizing flows, multiple coupling layers are composed to obtain a more complex transformation with the role of the two mappings in \eqref{eq:CF} swapped in alternating layers.
The Jacobian of \eqref{eq:CF} is a lower triangular matrix with a 2-by-2 block structure corresponding to the partition of $x$.
In \cite{realNVP17}, the coupling function $q$ is chosen as an affine function:
% \begin{equation*}
% \label{coupling_layer_formula}    
    $q(z) = z\cdot \exp(\alpha(x_{1:d}))+\beta(x_{1:d})$,
% \end{equation*}
where $\alpha$ and $\beta$ are arbitrary functions,
and the resulting flow is termed \emph{affine coupling flow}.

\begin{table*}[h]
\caption{Comparison of different normalizing flow architectures. \texttt{E.I.} stands for explicit inverse, \texttt{U.P.} for unrestricted parameters, \texttt{U.F.} for unrestricted function representations.}
    \label{tab:CF}
    \centering

    \begin{tabular}{c|cccc}
    \hline
     Method & monotone map & \texttt{E.I} & \texttt{U.P.} & \texttt{U.F.}  \\
     \hline
     %Real NVP \cite{realNVP17} & affine function & yes & yes & affine\\
     Real NVP \cite{realNVP17}   & affine function & yes & yes & no\\
     \hline
     Glow \cite{GLOW18} & affine function & yes & yes & no\\
      %Glow \cite{GLOW18} & affine function & yes & yes & affine\\
     \hline \\[-1em]
    %  Nonlinear squared flow\newline \cite{ziegler2019} &   $a+bx+\frac{c}{1+(dc+g)^2}$ & analytical & restricted & rational \\
    %  \hline
    % Flow++ \cite{flow++2019}  & $\alpha \sigma^{-1} \left( \sum\limits_{i=1}^r c_i \sigma\left( \frac{x-a_i}{b_i} \right) \right)+\beta$ & no & no & sigmoid \\[1.1em]
      Flow++ \cite{flow++2019}  & $\alpha \sigma^{-1} \left( \sum\limits_{i=1}^r c_i \sigma\left( \frac{x-a_i}{b_i} \right) \right)+\beta$ & no & no & no \\[1.1em]
     \hline 
    %  Cubic spline flow\newline \cite{cubicspline} & $\sigma^{-1} \circ h(x)\circ \sigma$ & analytical & restricted & spline \\
    %  \hline
     %NSF \cite{neuralspline19}  & rational-quadratic spline & yes & no & spline \\
      NSF \cite{neuralspline19} & rational-quadratic spline & yes & no & no \\
     \hline \\[-1em]
     %SOS \cite{sumsquare19} & $\int_0^x \sum\limits_{i=1}^L p_i(x)^2 dx + c$ & no & yes & polynomial \\[0.8em]
      SOS \cite{sumsquare19} & $\int_0^x \sum\limits_{i=1}^L p_i(x)^2 dx + c$ & no & yes & no \\[0.8em]
     \hline \\[-0.8em]
     %UMNN \cite{UMNN} & $\int_0^x f(x) dx + \beta \;\; (f>0)$ & no & yes & neural network \\[0.4em]
      UMNN \cite{UMNN} &  $\int_0^x f(x) dx + \beta \;\; (f>0)$ & no & yes & no  \\[0.4em]
     \hline \\[-0.85em]
    %  AUTM (new) & $x+\int_0^1 f(v(t),t) dt$ & analytical & unrestricted & general \\
    AUTM (new) & $x+\int_0^1 g(v(t),t)dt$ & yes & yes & yes \\[0.31em]
     \hline
\end{tabular}
\end{table*}

\section{Atomic Unrestricted Time Machine (AUTM) flows}

Lots of efforts have been made in recent years to construct a coupling function or transformer $q(x)$ that is strictly monotone (thus invertible) as well as expressive enough.
As shown in Table \ref{tab:CF}, sophisticated machinery is used to improve the expressive power and meanwhile ensure the monotonicity (invertibility) of $q$,
which usually requires restricting the form of $q$ or the model parameters, e.g. in \cite{ziegler2019,flow++2019,cubicspline,neuralspline19}.
Moreover, since $q$ is a complicated nonlinear function, an analytic format of  $q^{-1}$ is generally not available and thus numerical root-finding algorithms are often used to compute the inverse transformation.
It is natural to ask whether \emph{there exists a family of universal monotone functions with analytic inverses and unrestricted model parameters or representations?}

We propose a new approach to construct a monotone $q(x)$ based on integration with respect to a free \emph{latent} variable.
The introduction of the latent variable enables the use of unconstrained transformations and renders exceptional flexibility for manipulating the transformation and its inverse.
The resulting coupling flows and autoregressive flows can be inverted easily using the inverse formula and the Jacobian is triangular.
% which can not be achieved by existing integral-based methods.

% Motivated by the cumulative distribution function (\ref{Fi}), 
We define $q: \mathbb{R}\to \mathbb{R}$ through a latent function $v(t)$ by 
\begin{equation}
\label{eq:AUTM}
    \begin{aligned}
    q:x\to y = v(1), \quad v(t)=x+\int_0^t g(v(t),t)dt.
    \end{aligned}
\end{equation}
where $g(v,t)$ is uniformly Lipschitz continuous in $v$ and continuous in $t$ $(0\leq t \leq 1)$. Equivalently, $v(t)$ satisfies $v'(t)= g(v(t),t)$ and $v(0)=x$. So  the transformation from $x$ to $y$ can be viewed as an evolution of the latent dynamic $v(t)$. 
Note that the integral in (\ref{eq:AUTM}) is with respect to $t$ instead of $x$ and the integrand does \emph{not} have to be positive. 
Moreover, we can easily find the inverse transform as \begin{equation}
    \label{eq:AUTMinv}
    \begin{aligned}
    &q^{-1}: y\to x=v(0),\;
    v(t)=y+\int_1^t g(v(t),t) dt.
\end{aligned}
\end{equation}

An explicit inverse formula brings significant  computational speedups as compared to existing integral-based methods that rely on numerical root-finding algorithms. 
Compared with other coupling functions/transformers, there is no assumption on $g$ other than Lipschitz continuity.
We will show later in Section \ref{sec:theory} that the transformation in \eqref{eq:AUTM} is strictly increasing and is general enough to approximate \emph{any} given continuously increasing map.

\iffalse
In this section, we propose a new approach based on latent dynamics to resolve those issues.
Different from existing methods, which seek an \emph{explicit} expression of $q(x)$,
the proposed \emph{latent} framework does not start with an explicit formula that transforms $x$ to $y$.
Instead, we introduce a new \emph{latent} variable $t$ (thus a new dimension), which enables the use of more complex transformations and introduces exceptional flexibility for manipulating the transformation and its inverse.
%The monotonicity of the resulting transformation 
\fi

\paragraph{AUTM}
We term the mapping $q$ in \eqref{eq:AUTM} an ``Atomic Unrestricted Time Machine (AUTM)".
``\textbf{Atomic}" means that (i) $q$ is always \emph{univariate} and \emph{scalar-valued}; 
(ii) $g(v,t)$ can be as simple as an affine function in $v$ and does \emph{not} have to be a deep neural network to achieve good performance;
(iii) the computation of the transformation as well as Lipschitz constant is \emph{lightweight};
(iv) the model can be easily incorporated into existing normalizing flow architectures.
In fact, we will see later in Section \ref{sec:theory} that \emph{any} monotonic normalizing flow is a limit of AUTM flows.
``\textbf{Unrestricted}" means that there is no constraint on parameters or function forms in the model.
``\textbf{Time Machine}" refers to the fact that the model is automatically invertible and the computation of inverse is essentially a reverse of integral limits.
Thanks to the ``atomic" property, AUTM can be easily incorporated into triangular flow architectures such as coupling flows and autoregressive flows.

\paragraph{AUTM coupling flows}
Given a $D$ dimensional input $x=(x_1,\dots,x_D)$ and $d<D$,
the AUTM coupling layer $f:\mathbb{R}^D\to\mathbb{R}^D$ is defined as follows.
\begin{equation}
\label{eq:AUTM-CF}
\begin{split}
    y_{1:d}=x_{1:d},\quad
    y_{d+1:D}=q(x_{d+1:D};\theta(x_{1:d})),
\end{split}
\end{equation}
where $q$ is the AUTM map defined in \eqref{eq:AUTM}. 
%with  function $g(v,t;\theta(x_{1:d}))$ parameterized by  $\theta(x_{1:d})$. %to parametrize function $g$ in \eqref{eq:AUTM}.
% Here $q(x_{d+1:D};\theta(x_{1:d}))=(q(x_{d+1};\theta(x_{1:d})),\dots,q(x_D;\theta(x_{1:d})))$.
Let $f_*$ denote the AUTM coupling layer with $q$ applied to $x_{1:d}$ instead of $x_{d+1:D}$, i.e.
\begin{equation}
\label{eq:AUTM-CF2}
\begin{split}
    y_{1:d}=q(x_{1:d};\theta(x_{d+1:D})),\quad
    y_{d+1:D}=x_{d+1:D}.
\end{split}
\end{equation}
The AUTM coupling flow is defined by stacking $f$ and $f_*$
% \eqref{eq:AUTM-CF} and \eqref{eq:AUTM-CF2} 
in a multi-layer fashion as shown in Figure \ref{fig:autm}(left).
Since the inverse $q^{-1}$ is given in \eqref{eq:AUTM}, the inverse transformation $f^{-1}$  or $f_*^{-1}$ is readily available.

\paragraph{AUTM autoregressive flows}
An autoregressive flow is composed of autoregressive mappings, similar to coupling layers in coupling flows.
The AUTM autoregressive mapping on $\mathbb{R}^D$ is defined as
\begin{equation}
\label{eq:AUTM-AF}
\begin{split}
    f(x;\theta) = &(q_1(x_1;\theta_1),\dots, q_D(x_D;\theta_D(x_{1:D-1}))),
\end{split}
\end{equation}
where each $q_k(x_k;\theta_k(x_{1:k-1}))$ is an AUTM map with unrestricted conditioner $\theta_k(x_{1:k-1})$.
See Figure \ref{fig:autm}(right)  for an illustration.
% Note that there is \emph{no} restriction on the parameter set $\theta_k$
The inverse mapping $f^{-1}$ can be computed rapidly by first computing $q_1^{-1}$ (which gives $x_1$) and then $q_2^{-1}, q_3^{-1},\dots, q_D^{-1}$, where each $q_k^{-1}$ is explicitly given by \eqref{eq:AUTMinv}.
% \begin{equation}
% \label{eq:AUTM-AF2}
%     f^{-1}(y;\theta) = (q_1^{-1}(y_1;\theta_1),\dots,q_k^{-1}(y_k;\theta_k),\dots,q_D^{-1}(y_D;\theta_D)),
% \end{equation}
% where $q_k^{-1}$ is defined in \eqref{eq:AUTM}.
The Jacobian of $f$ in \eqref{eq:AUTM-AF} is lower triangular.

\begin{figure*}
    \centering
    \includegraphics[scale=0.6]{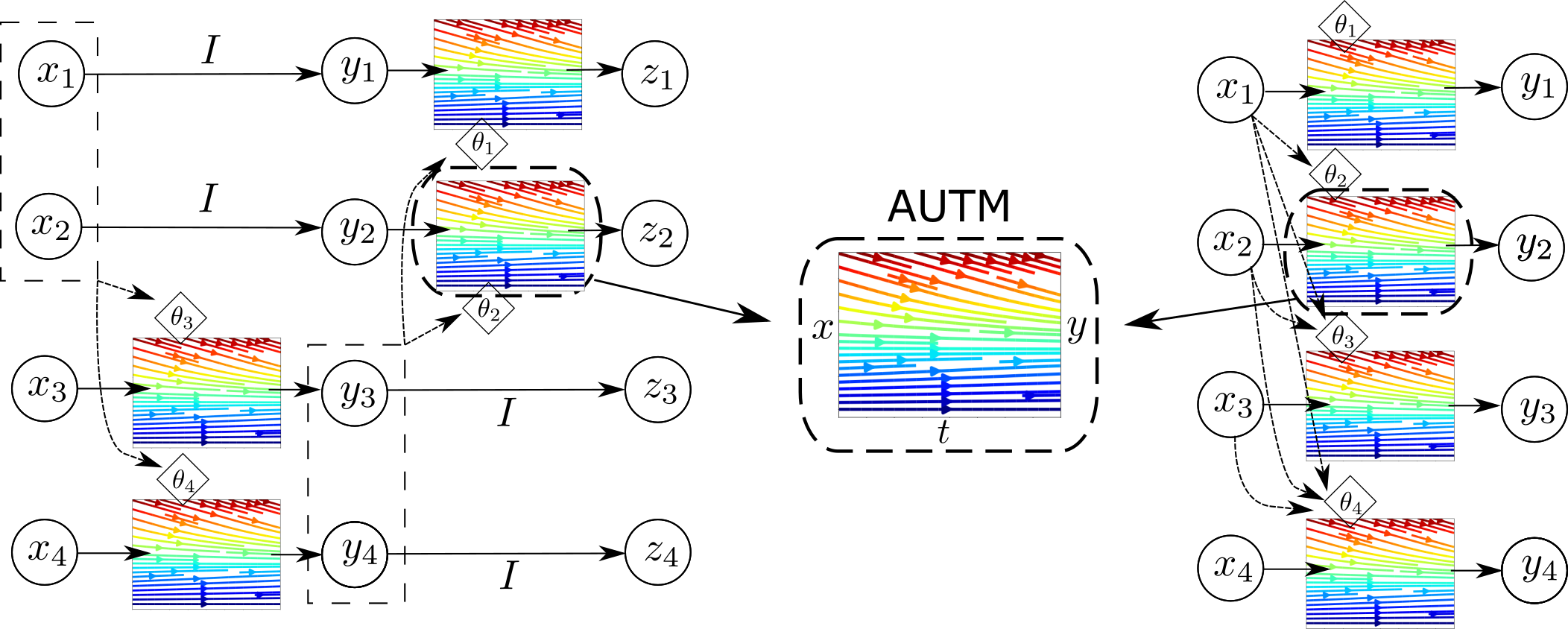}
    \caption{Left to right: AUTM coupling flow (2 layers), AUTM map, AUTM autoregressive flow (1 layer).}
    \label{fig:autm}
\end{figure*}

\iffalse
\begin{figure}
    \centering
    \begin{subfigure}[b]{0.5\textwidth}
                \includegraphics[width=\linewidth]{./images/autmCF.png}
                \caption{AUTM Coupling Flow (2 layers, $D=4$)}
                \label{fig:autmCF}
        \end{subfigure}
        \begin{subfigure}[b]{0.45\textwidth}
        \centering
                \includegraphics[width=0.6\linewidth]{./images/autmAF.png}
                \caption{AUTM Autoregressive Flow (1 layer, $D=4$)}
                \label{fig:autmAF}
        \end{subfigure}
%     \subfigure[AUTM Coupling Flow (2 layers, $D=4$)]{0.25\textwidth}{\label{fig:autmCF}
%     \includegraphics[scale=0.5]{./images/autmCF.png}
% }
% \hspace{8mm}
% \subfigure[AUTM Autoregressive Flow (1 layer, $D=4$)]{\label{fig:autmAF}
%     \includegraphics[scale=0.51]{./images/autmAF.png}
% }
    \caption{AUTM normalizing flows \cdf{perhaps mark the flow rectangle with AUTM? Also make it bigger so that it shows flows? -QY}}
    \label{fig:autm}
\end{figure}
\fi

\paragraph{Jacobian determinant and log-density}
\label{Jacobian_ours}
The Jacobian of AUTM flow is lower triangular.
It will be shown in Theorem \ref{thm:derivative} that the derivative of the mapping $q(x)$ is given by $q'(x) = \exp\left( \int_0^1\frac{\partial g}{\partial v}(v(t),t) dt \right)$.
Thus one can immediately derive the Jacobian determinant of AUTM flow.
If coupling layer is used, the Jacobian determinant is
\begin{equation*}
    %\begin{aligned}
    \exp\left( \int_0^1 \sum_{k=d+1}^D \frac{\partial g}{\partial v}(v_k(t),t;\theta(x_{1:d}))dt \right)\\
%    &\exp\left( \int_0^1 \sum_{k=1}^D \frac{\partial g}{\partial v}(v_k(t),t;\theta(x_{1:k-1}))dt \right).
%\end{aligned}
\end{equation*}
If autoregressive layer is used, the Jacobian determinant is
\begin{equation*}
    \exp\left( \int_0^1 \sum_{k=1}^D \frac{\partial g}{\partial v}(v_k(t),t;\theta(x_{1:k-1}))dt \right).
\end{equation*}

From the above formulas and \eqref{formula_of_flow}, the change of log-density of an AUTM flow follows immediately.
If coupling layer is used, then
\begin{equation*}
\log p_Y(y)=\log p_X(x)-\int_0^1 \sum_{k=d+1}^D \frac{\partial g}{\partial v}(v_k(t),t;\theta(x_{1:d})) dt.
\end{equation*}
If autoregressive layer is used, then
\begin{equation*}
\log p_Y(y)=\log p_X(x)-\int_0^1 \sum_{k=1}^D \frac{\partial g}{\partial v}(v_k(t),t;\theta(x_{1:k-1})) dt.
\end{equation*}

% \begin{equation*}
% \begin{aligned}
%     \text{AUTM coupling flow}:&\quad
%     \log p_Y(y)=\log p_X(x)-\int_0^1 \sum_{k=d+1}^D \frac{\partial g}{\partial v}(v_k(t),t;\theta(x_{1:d})) dt;\\
%     \text{AUTM autoregressive flow}:&\quad 
%      \log p_Y(y)=\log p_X(x)-\int_0^1 \sum_{k=1}^D \frac{\partial g}{\partial v}(v_k(t),t;\theta(x_{1:k-1})) dt.
% \end{aligned}
% \end{equation*}

% \paragraph{Backpropagation}
% \cdf{\textbf{Backpropagating???} using adjoint ODE or direct calculation as in coupling flow ?}

% But there will be an unacceptable error when use traditional ODE solver to calculate the inverse in our ODE coupling layer. We have known that if we use ODE solver after several coupling layers, the error will be enlarged to a unknown number. If we directly use ODE solver on our ODE coupling layer for both forward and inverse, we will see that $x$ is completely different from $f^{-1}(f(x))$, unless use a lot of points. As a result, we propose a method to calculate the inverse.

\iffalse
\cdf{Assume that, after use a specific numerical ODE solver, formula \eqref{eq:AUTM-CF} becomes 
\begin{equation}\label{numerical_coupling_formula_ODE_solver}
\begin{split}
    y_{1:d}&=x_{1:d}\\
    y_i&=L(x_i,x_{1:d};\theta)
\end{split}
\end{equation}
for $i=d+1,...,D$. In order to get the inverse, we use the same numerical ODE solver to calculate a number, we called it 'Initial Guess'. When we get the initial guess, we use numerical methods like fixed point iteration to compute the inverse rapidly. 
The inverse can be computed accurately up to machine precision.}
\fi

\section{Monotonicity and universality of AUTM flows}
\label{sec:theory}
In this section, we present several key results on AUTM flows, including monotonicity and universality. The proofs can be found in Appendix \ref{app:proof}.
The derivative of $q(x)$ in \eqref{eq:AUTM} is explicitly available.
%\begin{restatable}[Derivative]{thm}{derivative}
\begin{theorem}[Derivative]
\label{thm:derivative}
    Let $q(x)$ be defined in \eqref{eq:AUTM}. Then 
        $ q'(x) = \exp\left( \int_0^1\frac{\partial g}{\partial v}(v(t),t) dt \right)$. 
 \end{theorem}
%\end{restatable}

%It is easy to see that $q'>0$, so $q$ is strictly increasing.
% We term the monotonicity of $q$ ``latent" since it is derived via the latent variable $t$ and is hard to infer from the analytic expression in \eqref{eq:AUTM}.
\begin{theorem}[Monotonicity]
%\begin{restatable}[Monotonicity]{thm}{mono}
\label{thm:mono}
    The mapping $q(x)$ defined in \eqref{eq:AUTM} is invertible and  strictly increasing.
\end{theorem}
%\end{restatable}

The expressive power of AUTM is summarized in the following theorems, which state that one can approximate \emph{any} monotone continuous transformation with a family of AUTM transformations. 
% In Section \ref{sec:theory}, we give a concrete example of approximating \emph{any} given monotone flow with a family of AUTM flows.
% The construction is extremely general and provides a variety of new ways on the design of normalizing flows.

%\begin{restatable}[AUTM as a universal monotone mapping]{thm}{dense}
\begin{theorem}[AUTM as a universal monotone mapping]
\label{thm:dense}
Let $\mathcal{C}$ be the space of continuous functions on $\mathbb{R}$ with compact-open topology and let $\mathcal{M}\subset \mathcal{C}$ be the cone of (strictly) increasing continuous functions.
Then the set of AUTM bijections 
$$
	\mathcal{Q} = \{ q(x) \text{ in } \eqref{eq:AUTM}: v(0)=x\in\mathbb{R} \}
$$
is dense in $\mathcal{M}$.
\end{theorem}
%\end{restatable}

%\begin{restatable}[AUTM as a universal flow]{thm}{convergence}
\begin{theorem}[AUTM as a universal flow]
\label{thm:convergence}
For any coupling or autoregressive flow
$F=F_1\circ F_2\circ\cdots\circ F_p$ from $\mathbb{R}^D$ to $\mathbb{R}^D$,
where each $F_k$ is a triangular monotone transformation,
there exists a family of AUTM flows
$\{ T_s \}_{s>0}=\{ T_{s,1}\circ T_{s,2}\circ \cdots\circ T_{s,p} \}_{s>0}$
such that $T_s$ converges to $F$ pointwisely and compactly as $s\to 0$.
In fact, there exists a family such that the convergence rate is $O(e^{-\frac{1}{s}})$ as $s\to 0$.
\end{theorem}

Theorem \ref{thm:dense} and Theorem \ref{thm:convergence} imply that all coupling flows and autoregressive flows can be approximated arbitrarily well by AUTM flows.
In the following, we present an explicit construction of such a family of AUTM flows that converge to an arbitrarily given monotonic normalizing flow.
The convergence result provides a link between the proposed AUTM flows and existing monotonic normalizing flows and illustrates the representation power of AUTM.
Notice that every AUTM flow has explicit inverse, so the universality result in this section shows that we can approximate any \emph{monotonic} flow by a flow with \emph{explicit} inverse.
% The universality and flexibility facilitate straightforward derivation of new normalizing flows by applying the AUTM framework to existing architectures.
% The resulting flows all have explicit inverses and no model constraint.

\paragraph{Universal AUTM flows.}
Let $\phi(x)$ be an arbitrary increasing continuous function on $\mathbb{R}$.
Define a family of AUTM bijections parametrized by $s>0$ as follows:
\begin{equation}
\label{eq:AUTMqs}
    q_s(x) = x+\int_0^1 \phi(v_s(te^{-\frac{1}{s}})) - v_s(te^{-\frac{1}{s}}) dt,
\end{equation}
where $v_s(t) = x+\int_0^t \phi(v_s(ze^{-\frac{1}{s}})) - v_s(ze^{-\frac{1}{s}}) dz$.
Then it can be shown that (see Appendix \ref{app:proof}  - Proof of Theorem \ref{thm:dense}):
$q_s|_K$ converges to $\phi|_K$ uniformly on any compact set  $K\subset\mathbb{R}$ as $s\to 0$
and the convergence rate is $O(e^{-\frac{1}{s}})$.
Based on $q_s$, one can construct a family of AUTM coupling flows or autoregressive flows that converge to the given flow based on $\phi$.

In fact, the family of AUTM flows in \eqref{eq:AUTMqs} is just one particular family of universal AUTM flows with more general integrands.
This is formalized in the theorem below regarding universal AUTM flows that generalize \eqref{eq:AUTMqs}. The proof is given in Appendix \ref{app:proof}.

%\begin{restatable}[General families of universal AUTM flows]{thm}{family}
\begin{theorem}
\label{thm:family}
    For $s>0$,
    let $\kappa_s \in C([0,1])$ be a positive function such that 
    \begin{equation}
\label{eq:kernelConditions}
 \int_0^1 \kappa_s(t) dt =1 \text{ and } \int_0^{e^{-\frac{1}{s}}} \kappa_s(t) dt\to 0 \; \text{as}\; s\to 0.
\end{equation}
    Given any increasing continuous function $\phi(x)$,
    we define $q_s$ as follows
    % \begin{equation*}
        $q_s(x) = x+\int_0^1 g_s(v_s,t) dt$,
    % \end{equation*}
    where $g_s(v,t)=\kappa_s(t)[\phi(v(te^{-\frac{1}{s}}))-v(te^{-\frac{1}{s}})]$ and $v_s(t) = x+\int_0^t g_s(v_s,z) dz.$
        %  q_s(x) = x+\int_0^1 \kappa_s(t) [\phi(v_s(t))-v_s(t)] dt\quad\text{with}\quad v_s(t) = x+\int_0^t \kappa_s(t) [\phi(v_s(t))-v_s(t)] dt.
    % with 
    % \begin{equation*}
    % v_s(t) = x+\int_0^t \kappa_s(t) [\phi(v_s(t))-v_s(t)] dt.
    % \end{equation*}
    Then as $s\to 0$, $q_s|_K$ converges to $\phi|_K$ uniformly for any compact set  $K\subset\mathbb{R}$.
%\end{restatable}
\end{theorem}

\textit{Remark 1.}
The proof of Theorem \ref{thm:dense} in Appendix \ref{app:proof}
indicates that it may be sufficient to choose $g(v,t)$ as a function that is explicit in $v$ only.
In fact, it is shown in  \eqref{eq:dvdtau} that, after a scaling of the time variable, the equation for the approximant $v_s$ is autonomous.
Thus we expect good approximation power if $g(v,t)=\frac{dv}{dt}$ is explicit in $v$ only.

\section{Related work}

\paragraph{Integral-based methods: SOS and UMNN.}
Existing integral-based methods like \cite{sumsquare19,UMNN} require the integrand to be positive and the inverse transformation is \emph{not} analytically available. 
When computing the inverse transformation, AUTM only requires evaluating \emph{one} integral while above methods require evaluating $k$ different integrals for $k$ iterations in the root-finding algorithm.
Moreover, \cite{UMNN} models the integrand as a positive neural network while AUTM allows using general function classes with better computational efficiency than neural networks.

\paragraph{Neural ODE and FFJORD.}
Neural ODE \cite{NeuralODE18} and Free-form Jacobian of Reversible Dynamics (FFJORD) \cite{FFJORD} use a dynamical system 
% $\frac{dv}{dt}=g(v,t;\theta)$
to model the transformation, in which a \emph{multivariate} neural network is used to model the \emph{vector-valued} dynamics.
The use of the integral representation in AUTM to enable the analytic inverse transformation is inspired by Neural ODE and FFJORD.
AUTM differs from FFJORD in three aspects. 
{Firstly}, FFJORD is not computationally-efficient because it relies on the the neural network to model the dynamics and the Jacobian is a fully dense matrix whose log-determinant can not be computed easily.
AUTM, similar to other integral-based flows, employs (1) a coupling or autoregressive structure so that the Jacobian is a triangular matrix and the log-density is explicitly given;
(2) decoupled entrywise transformations in each layer.
The integrand $g$ in AUTM is specified by the user and can be as simple as polynomials while still achieve competitive results.
{Secondly}, it is rigorously proved that AUTM admits universal approximation property, while it is unclear weather FFJORD admits universality.
{Thirdly}, AUTM is a monotonic flow while no monotonicity result for FFJORD can be found.
%, which will require a lot of samples to reach the desired accuracy if the variance of the target random variable is large.
%Moreover, the fully coupled complicated transformation is very likely to yield a large Lipschitz constant that affects the numerical stability during training.
%Our proposed architecture has no restriction on the function classes and allows the use of simple \emph{univariate} functions like polynomials to model the dynamic.
%This leads to fewer function evaluations as well as improved robustness.
% Another major advantage is that the Jacobian of the transformation is triangular and thus can be computed easily.
Overall, AUTM benefits from the triangular Jacobian and decoupled transformations with simple integrands and is thus much more computationally efficient than FFJORD.
It is also possible to incorporate hierarchical structures in \cite{smash,IPDPS2020} to further improve the efficiency of AUTM, which will be pursued in a later date.

\paragraph{Other models.}
% Coupling flows split the input vector into two parts and only transform the second part using a coupling layer.
% Such a transformation can be repeated in a multi-layer fashion with swapped coupling layer each time.
Affine coupling flows \cite{nice2014,realNVP17} use an affine coupling function so that the inverse is trivial to compute.
Recent developments consist in using \emph{nonlinear} monotone functions to improve the expressiveness of affine coupling functions.
To guarantee invertibility, unlike integral-based models, many architectures e.g.  \cite{ziegler2019,neuralspline19,naf2018,bnaf20} restrict model parameters, which  limit the expressive power of the model and the training efficiency.
Moreover, computation of the inverse transformation usually requires numerical root-finding methods since a tractable analytic inverse is often \emph{not} available.
The AUTM framework circumvents those issues by using an integral representation with respect to a \emph{latent} variable.
The inverse is explicit, regardless of the choice of model classes or parameters.
This enables the rigorous justification of the universality of AUTM flows.

%%%%%%%%%%%%%%%%%%%%%%%%%%%%%%%%%%%%%%%%%%%%%%%%%%%%%%%%
\section{Experiment}
In this section, we present experiments to evaluate our model. In Section \ref{section_density_estimation}, we perform density estimation on five tabular datasets and compare with other methods. In Section \ref{section_image_dataset}, we train our model on the CIFAR10 and ImageNet32 datasets for image generation. Experiments are conducted on either Nvidia 3080 GPU or Nvidia V100 GPU. All experimental details are provided in  Appendix \ref{sec:ExpAppendix}. 

For image datasets, we model $g(v,t)$ in \eqref{eq:AUTM} as a quadratic polynomial in $v$. 
For density estimation, we consider three different choices of $g(v,t)$ in \eqref{eq:AUTM}: $g(v,t)=av+b+cv^2$, $g(v,t)=av+b+cv^3$ and $g(v,t)=av+b+c\sigma(v)$, where $\sigma$ denotes the sigmoid function ($g$ is chosen to be explicit only in $v$ due to \textit{Remark 1}).
Note that this is different from many existing methods that rely on deep neural networks to model the core function in the model, such as UMNN \cite{UMNN} for the positive integrand, NAF \cite{naf2018} and BNAF \cite{bnaf20} for the autoregressive mapping, neural ODEs \cite{NeuralODE18} and FFJORD \cite{FFJORD} for the entire dynamical system.
We show in the following that, compared to the state-of-the-art models, our proposed AUTM model achieves excellent performance with simple choices of $g$. 
More importantly, for high-dimensional image datasets like ImageNet32, AUTM model requires significantly less model parameters compared to other models.

\subsection{Density Estimation}\label{section_density_estimation}

\textbf{Data sets and baselines.} We first evaluate our method on four datasets from the UCI machine-learning repository \cite{UCIdata}: POWER, GAS. HEPMASS, MINIBOONE, and also the BSDS300 dataset, which are all preprocessed by \cite{MAF17}. We compare our method to several existing normalizing flow models, including Real NVP \cite{realNVP17}, Glow \cite{GLOW18}, RQ-NSF \cite{neuralspline19}), CP-FLOW \cite{cpflow21}, FFJORD \cite{FFJORD}, UMNN \cite{UMNN} and autoregressive models such as MAF  \cite{MAF17}, MADE \cite{made15} and BNAF \cite{bnaf20}. 

\textbf{Model configuration and training.} We use 10 (or 5) masked AUTM layers and set the hidden dimensions 40 times (or 10 times) the dimension of the input. We apply a random permutation of the elements of each output vector, as the masked linear coupling layer, so that a different set of elements is considered at each layer, which is a widely used technique \cite{bnaf20}, \cite{realNVP17}, \cite{MAF17}. We use Adam as the optimizer and select hyperparameters after an extensive grid search. 
% Detailed training settings can be found in Appendix \ref{sec:ExpAppendix}. 
% The optimizer is Adam with Polyak averaging, the parameter is 0.998. We also applied an exponential learning rate decay with rate 0.5 when the performance in the validation set does not improve for several epochs, and the minimum learning rate is set as $0.0005$. 

\textbf{Results.} We report average negative log-likelihood estimates on the test sets in Table \ref{experiment_5data}. It can be observed that AUTM consistently outperforms Real NVP, Glow, MADE, MAF, CP-Flow. On MINIBONDE dataset, our models perform better than all other models except BNAF. On POWER, HEPMASS, BSDS300 dataset, one of the AUTM models performs best among all baseline models. On GAS dataset, our results are competitive at either top 2 or top 3 spot with a tiny gap from the best.

\begin{table*}[ht]
  \small
  \centering
  \caption{Average test negative log-likelihood (in nats) of tabular datasets (lower is better). Numbers in the parenthesis are standard deviations.
  Average/standard deviation is computed by 3 runs. 
  The best performance for each dataset is highlighted in boldface.}
  \label{experiment_5data}
  \vskip.05in
      \renewcommand{\arraystretch}{1.3}
  \begin{tabular}{c|ccccc}
    \toprule
    Model & POWER &  GAS &  HEPMASS & MINIBOONE & BSDS300\\
    \midrule
    Real NVP\cite{realNVP17}   & -0.17(0.01) & -8.33(0.14) & 18.71(0.02) & 13.55(0.49) & -153.28(1.78)\\
    Glow\cite{GLOW18}          & -0.17(0.01) & -8.15(0.40) & 18.92(0.08) &  11.35(0.07) & -155.07(0.03)\\
    FFJORD\cite{FFJORD}        & -0.46(0.01) & -8.59(0.12) & 14.92(0.08) & 10.43(0.04) & -157.40(0.19)\\
    UMNN\cite{UMNN}            & -0.63(0.01) & -10.89(0.70) & \textbf{13.99(0.21)} & 9.67(0.13) & \textbf{-157.98(0.01)}\\
    MADE\cite{made15}       & 3.08(0.03)  & -3.56(0.04) & 20.98(0.02) & 15.59(0.50) &  -148.85(0.28)\\
    MAF\cite{MAF17}        & -0.24(0.01) & -10.08(0.02) & 17.70(0.02) & 11.75(0.44) & -155.69(0.28)\\
    CP-Flow\cite{cpflow21}    & -0.52(0.01) & -10.36(0.03) & 16.93(0.08) & 10.58(0.07) & -154.99(0.08)\\
    BNAF\cite{bnaf20}      & -0.61(0.01) & -12.06(0.09) & 14.71(0.38) & \textbf{8.95(0.07)} & -157.36(0.03)\\
%    Q-NSF (C)\cite{neuralspline19}  & -0.64(0.01)  & -12.80(0.02) & 15.35(0.02) & 9.35(0.44) & -157.65(0.28)\\
    RQ-NSF (C)\cite{neuralspline19} & \textbf{-0.64(0.01)}  & \textbf{-13.09(0.02)} & 14.75(0.03) & 9.67(0.47) & -157.54(0.28)\\
    \midrule
    AUTM: $g(v,t)=av+b+cv^2$ & -0.63(0.03) & -12.24(0.04) & 14.62(0.30)
    & 9.16(0.18) & -157.45(0.05)\\
    AUTM:  $g(v,t)=av+b+cv^3$ & \textbf{-0.64(0.01)} & -12.37(0.06) & 14.76(0.25) & 9.33(0.10) & -157.54(0.10)\\
    AUTM: $g(v,t)=av+b+c\sigma(v)$ & -0.61(0.02) & -12.03(0.06) & 14.94(0.33) & 9.29(0.20) & -157.28(0.14)\\
    \bottomrule
  \end{tabular}
\end{table*}

\subsection{Experiment on image dataset}\label{section_image_dataset}

\textbf{Data sets and baselines.} We then evaluate our method on the CIFAR10 \cite{cifar10} and ImageNet32 \cite{imagenet32} datasets. Unlike density estimation tasks, image datasets are large-scale and high-dimensional. As a result, there are only a limited number of models available for image tasks. We calculate bits per dim and compare with other normalizing flow models including Real NVP \cite{realNVP17}, Glow \cite{GLOW18}, Flow++ \cite{flow++2019}, NSF \cite{neuralspline19}. 
 The results of bits per dim for each model are given in Table \ref{result_image_table}. We also show sampled images by using AUTM in Figure \ref{figure_image}. 

\textbf{Model configuration and training.} We use 14 AUTM coupling layers with 8 residual blocks for each layer (cf. \cite{deepresiduallearning16}) in our model. Each residual block has three convolution
layers with 128 channels. Our method is trained for 2500 epochs with batch size 64 for CIFAR10, and 50 epochs with batch size 64 for ImageNet32 dataset. 

\textbf{Results.} From Table \ref{result_image_table}, we can find our method outperforms all baselines with the exception of Flow++ \cite{flow++2019} on CIFAR10 dataset, which uses a variational dequantization technique.
In addition, Table \ref{result_image_table} shows AUTM is not sensitive to the size of the dataset when we transfer from CIFAR10 to ImageNet32. Comparing the number of parameters used for each method, we find that AUTM yields the best performance with much fewer parameters compared to other models for ImageNet32. In particular, the number of model parameters of Flow++ increases dramatically as we move from CIFAR10 to ImageNet32 while \emph{affine} models like Real NVP and Glow only have a moderate increase in the number of parameters. This is reasonable since Flow++ is the only \emph{nonlinear} model other than AUTM. It demonstrates that AUTM, as a \emph{nonlinear} model, yields better efficiency and robustness in terms of parameter use.

\begin{table}[ht]
  \small
  \centering
  \caption{Results of BPD (bits per dim) on CIFAR10 and ImageNet32 datasets. Results in brackets indicate the model using variational dequantization.}
  \label{result_image_table}
     \vskip.05in
     \setlength\tabcolsep{1.5pt}
  \begin{tabular}{c|cccc}
    \toprule
     & CIFAR10 & CIFAR10 &  ImageNet32 & ImageNet32\\
     Model & BPD & parameters & BPD & parameters\\
    \midrule
    Real NVP   & 3.49 & 44.0M & 4.28 & 66.1M \\
    Glow       & 3.35 & 44.7M &4.09 & 67.1M \\
    Flow++ & (3.08) & 31.4M & (3.86) & 169.0M\\
    RQ-NSF (C) & 3.38 & 11.8M & - & -\\
    \midrule
    Our Method & 3.29 & 35.5M & 3.80 & 35.5M\\
    \bottomrule
  \end{tabular}
\end{table}

\begin{figure*}[ht]
  \centering
  \includegraphics[width=.45\linewidth]{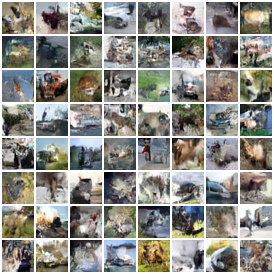}
  \hspace{1cm}
  \includegraphics[width=.45\linewidth]{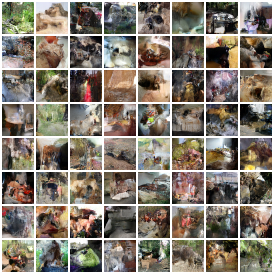}
  \caption{\textbf{Left}: Samples generated by using a pretrained model on CIFAR10 dataset. \textbf{Right}: Samples generated by using a pretrained model on ImageNet32 dataset.}
  \label{figure_image}
\end{figure*}

\subsection{Numerical inversion of AUTM}
Existing normalizing flow models with no explicit inverse usually employ bisection to compute the inverse transformation.
For AUTM, more options are available to compute the inverse mapping, such as fixed point iteration, which offers faster convergence than bisection.
We compare the performance of bisection and fixed point iteration for AUTM by considering a toy example where function $g$ in \eqref{eq:AUTMinv} is chosen as a specific quadratic polynomial in $v$ and the input variable $x$ is randomly chosen from the unit interval.
We use the discretized version (five-point trapezoidal rule) of the inverse formula in \eqref{eq:AUTMinv} as the initial guess for fixed point iteration.
Table \ref{result_iteration_table} shows that this leads to significantly fewer (around 50\%) iteration steps than bisection to achieve the same solution accuracy.

% \color{red}
% Traditional way to find the inverse of a monotonic flow layer is using binary search. Besides, for an ODE flow layer, people use the integral to integrate their functions from time $T$ to $0$. However, in our AUTM flow layer, we use the discrete form of the ODE as the forward layer, and use fixed-point iteration method to find the exact the inverse. 

% To get the inverse of the given data $y$, first calculate the discrete form of the ODE inversely to get an initial guess $x_0$. Then, calculate $x_{n+1}=x_n+(y-f(x_n))$ until we find an acceptable $x_N$. 

% \begin{restatable}[Convergence rate of the fixed point iteration]{thm}{YJ_1}
% %\begin{theorem}[AUTM as a universal flow]
% The convergence rate of the fixed point iteration is linear if $f'(x) \neq 0$.
% \end{restatable}

% We deploy an experiment to compare the number of the steps needed for iteration method and the binary search method to get the inverse. Suppose the input is a random number sampled from $(0,1)$, the forward layer is a 5-step discrete form of ODE $\frac{dx}{dt}=ax^2+bx+c$, where $a$ is randomly sampled from $(0,0.1)$, $b$ is randomly sampled from $(0,0.5)$, $c$ is randomly sampled from $(0,1)$. The experiment will calculate the output from the ODE first, and give the output to the iteration method and the binary search method to let the two methods calculate the inverse and compare the steps the two methods need. We run the experiment 1000 times for different error tolerance, and the results are listed in Table \ref{result_iteration_table} .

\begin{table}[ht]
  \small
  \centering
  \caption{Number of steps (averaged over 1000 random input) of root-finding method to reach a certain error tolerance}
  \label{result_iteration_table}
     \vskip.05in
  \begin{tabular}{c|cccc}
    \toprule
     Error tolerance & 1e-3 & 1e-4 &  1e-5 & 1e-6\\
    \midrule
    Iteration Method   & 4.565 & 6.642 & 8.658 & 10.831 \\
    Binary Search       & 8.967 & 12.398 & 15.668 & 19.073 \\
    \bottomrule
  \end{tabular}
\end{table}

\subsection{Comparison of FFJORD and AUTM}
Next, we test the runtime of FFJORD and AUTM on four datasets from the UCI machine-learning repository POWER, GAS. HEPMASS, MINIBOONE, all preprocessed by \cite{MAF17}. 
For AUTM, we choose $g(v,t)=av+b+cv^3$. We define the target negative log-likelihood (target NLL) as the NLL achieved by FFJORD after training for 12 hours. 
The time for each method to reach the target NLL is reported in Table \ref{result_comparison_FFJORD}.
It demonstrates that AUTM is significantly more efficient than FFJORD.
This is attributed to the structural advantages of AUTM. Firstly, AUTM transforms the input vector $x\in\mathbb{R}^D$ in an entrywise fashion where the $i$th entry is a univariate function of $x_i$. In FFJORD, the transformation of $x$ is characterized by a neural network where each output dimension is a nonlinear multivariate function of $x=(x_1,\dots,x_D)$.
Secondly, due to the aforementioned structural differences, AUTM has a triangular Jacobian while FFJORD has a dense Jacobian that induces difficulty in computing the log-determinant accurately.
Thirdly, the integrand $g$ in AUTM can be chosen as a simple function, for example, a quadratic function in $v$. In FFJORD, the integrand needs to be a neural network with $D$ input variables $x_1,\dots,x_D$. To evaluate the integral of such a complicated integrand accurately, a large number of quadrature nodes are needed, which will increase the  cost in both forward and backward transformations.
Additionally, AUTM enables the use of user-defined integrand $g$, which will be beneficial if prior information of the transformation to be learned is available.

% \begin{table}[ht]
%   \small
%   \centering
%   \caption{The time needed for FFJORD and AUTM to arrive the target negative log-likelihood for each dataset}
%   \label{result_comparison_FFJORD}
%      \vskip.05in
%   \begin{tabular}{c|ccccc}
%     \toprule
%      Dataset & POWER & GAS &  HEPMASS & MINIBOONE & BSDS300\\
%      Target NLL & 0.23 & -5.24 & 21.85 & 11.29 & 14.88\\
%     \midrule
%      FFJORD   & 12 hours & 12 hours & 12 hours & 12 hours & 12 hours \\
%      AUTM     & \textbf{6.92} minutes &  \textbf{3.67} minutes  &  \textbf{7.40} minutes & \textbf{1.75} minutes & \textbf{1 hours}(1 epoch)\\
%     \bottomrule
%   \end{tabular}
% \end{table}

\begin{table}[ht]
  \small
  \centering
  \caption{Runtime for FFJORD and AUTM to reach the target negative log-likelihood for each dataset.}
  \label{result_comparison_FFJORD}
     \vskip.05in
  \begin{tabular}{c|c|cc}
    \toprule
     Dataset & Target NLL & FFJORD & AUTM\\
    \midrule
     POWER     & 0.23  & 12hr & 6.92min\\
     GAS       &-5.24  & 12hr & 3.67min\\
     HEPMASS   & 21.85 & 12hr & 7.40min\\
     MINIBOONE & 11.29 & 12hr & 1.75min\\
    %  BSDS300   & 14.88 & 12hr & 1.00min\\
    \bottomrule
  \end{tabular}
\end{table}

\color{black}

%%%%%%%%%%%%%%%%%%%%%%%%%%%%%%%%%%%%%%%%%%%%%%%%%%%%%%%%
\section{Summary}
We have introduced a new nonlinear monotonic triangular flow called AUTM. 
% AUTM bears an integral form over a latent variable and enjoys guaranteed monotonicity and explicit inverse without restricting the representation of the integrand and parameters.
AUTM leverages the explicit inverse formula used in FFJORD and the triangular Jacobian in coupling and autoregressive flows.
Compared to FFJORD, AUTM demonstrates much better computational efficiency thanks to the triangular Jacobian structure, decoupled input dimensions, simple representation of the scalar integrand.
%and scales well for high-dimensional datasets.
Compared to other monotonic flows, AUTM has unrestricted parameters and more convenient computation of the inverse transformation.
% The numerical inversion of AUTM is more efficient than other integral-based architectures.
% AUTM can be incorporated seamlessly into existing flow architectures.
Theoretically, we have proved that AUTM is a universal approximator for any monotonic normalizing flow.
The performance is demonstrated by comparison to the state-of-the-art models in density estimation and image generation. As a nonlinear monotonic flow, AUTM is able to achieve competitive performance on high-dimensional image datasets.

%\paragraph{Limitations}
%One common issue of integral-based %methods is the cost and numerical error in computing the integral, which may limit the number of layers used in the architecture. 
%In UMNN, a sophist%%icated quadrature %scheme is discussed to improve the performance. For AUTM, Newton's method can be used for faster convergence with initial guess calculated by discretizing the inverse formula.
% As a result, we may need to use more discretization points to control the approximation errors.
%
% Unlike SOS and UMNN, only one integral needs to be evaluated in the numerical inversion of of AUTM.
%Another issue is the suitable choice of $g(v,t)$ that best balances the expressiveness and efficiency. We can model $g(v,t)$ as a neural network for expressiveness, but this will lead to increased costs. According to the experiments, simple functions already produce promising results. 
%An interesting problem is to study the expressiveness of AUTM with respect to the choice of $g$.
%Another limitation lies in the numerical stability challenges when implementing mathematically optimal family of functions.    

% \subsubsection*{Acknowledgements}
% All acknowledgments go at the end of the paper, including thanks to reviewers who gave useful comments, to colleagues who contributed to the ideas, and to funding agencies and corporate sponsors that provided financial support. 
% To preserve the anonymity, please include acknowledgments \emph{only} in the camera-ready papers.

\bibliography{autm}
\bibliographystyle{plain}

\clearpage
\appendix

\section{Proofs}
\label{app:proof}

% \derivative*
\begin{proof}[Proof of Theorem \ref{thm:derivative}]
    Write $v=v(t,x)$ as a function of $t$ and $x$. In fact,
    \[
        v(t,x) = x+\int_0^t g(v(t,x),t) dt.
    \]
    Define $u(t)=\frac{\partial v}{\partial x}$. It follows from the formula of $v$ that
    \begin{equation}
    \label{eq:ut}
        u(t) = \frac{\partial v}{\partial x} = 1+\int_0^t \frac{\partial g}{\partial v}\frac{\partial v}{\partial x} dt.
    \end{equation}
    Then we see that 
    \[
        \frac{du}{dt} = \frac{\partial g}{\partial v}\frac{\partial v}{\partial x} = \frac{\partial g}{\partial v} u.
    \]
    This implies that 
    \[
        u(t) = C\exp\left(\int_0^t \frac{\partial g}{\partial v} dt\right).
    \]
    Recall from \eqref{eq:ut} that $u(0)=1$, so $C=1$.
    Since $q(x)=v(1,x)$, we now conclude that 
    \[
        q'(x) = \frac{\partial v}{\partial x}(1,x) = u(1) = \exp\left(\int_0^1 \frac{\partial g}{\partial v} dt\right).
    \]
\end{proof}

% \mono*
\begin{proof}[Proof of Theorem \ref{thm:mono}]
    Obviously, monotonicity implies invertibility, so it suffices to show that $q$ is strictly increasing.
    There are several ways to prove this.
    The simplest way is to use Theorem \ref{thm:derivative} to see that $q'(x)>0$, so $q$ must be increasing.
    Below we present a different proof without using the analytic expression of $q'(x)$.

    Let $v_x(t)$ denote the function in \eqref{eq:AUTM} with $v(0)=x$,
    where $g(v,t)$ is continuous in $t$ and uniformly Lipschitz continuous in $v$.
    We need to show that for any $x<x'$, there holds $q(x)<q(x')$.
    We prove this by contradiction.
    Assume that there exist $x<x'$ such that $q(x)\geq q(x')$.
    There are two cases to consider:
    $q(x)=q(x')$ and $q(x)>q(x')$.
    \paragraph{Case 1:} $q(x)=q(x')=C$ for some constant $C$.
    \newline
    In this case, $v_x(1)=v_{x'}(1)=C$.
    Define $w_a(t)=v_a(1-t)$ for any $a\in\mathbb{R}$ and $t\in [0,1]$.
    Then it is easy to see that $w_x(t)$, $w_{x'}(t)$ are both solutions to the ODE
    \begin{equation}
    \label{eq:dwdt}
        \frac{dw}{dt} = -g(w(t),1-t),\quad w(0)=C,\quad t\in [0,1].
    \end{equation}
    Note that $w_x(t)$ and $w_{x'}(t)$ are two different solutions of \eqref{eq:dwdt} because $w_x(1)=x<x'=w_{x'}(1)$.
    This contradicts the uniqueness of solution to the ODE (which is well-posed since $g$ is Lipschitz) and we conclude that 
    the assumption $q(x)=q(x')$ can \emph{not} hold.
    
    \paragraph{Case 2:} $q(x)>q(x')$.
    \newline
    In this case, we have $v_x(1)>v_{x'}(1)$ and $v_x(0)<v_{x'}(0)$.
    Applying intermediate value theorem to $v_x(t)-v_{x'}(t)$ yields that 
    there exists $\tau\in (0,1)$ such that 
    \[
        v_x(\tau)=v_{x'}(\tau)=C
    \]
    for some constant $C$.
    Similar to Case 1, if we define $w_a(t)=v_a(\tau-t)$ for $t\in [0,\tau]$,
    then we can deduce that the ODE
    \begin{equation}
    \label{eq:dwdt2}
        \frac{dw}{dt}=-g(w(t),1-t),\quad w(0)=C
    \end{equation}
    has two different solutions $w_x(t)$ and $w_{x'}(t)$ as 
    $w_x(\tau)=x<x'=w_{x'}(\tau)$,
    which contradicts the well-posedness of \eqref{eq:dwdt2}.
    
    We now conclude that the inequality $q(x)\geq q(x')$ can \emph{not} hold.
    Consequently, $q(x)$ must be strictly increasing and the proof is complete.
\end{proof}

%\dense*
\begin{proof}[Proof of Theorem \ref{thm:dense}]
Since the set of increasing Lipschitz continuous functions is dense in $\mathcal{M}$, it suffices to consider Lipschitz functions in $\mathcal{M}$.

We need to show that,
given an arbitrary increasing Lipschitz continuous function $\phi(x)$,
there exists a family of AUTM bijections $\{ q_s(x) \}_{s>0}\subset \mathcal{Q}$ that converge compactly to $\phi(x)$ as $s\to 0$,
i.e. $q_s|_K\to \phi|_K$ uniformly on any compact set $K\subset \mathbb{R}$ as $s\to 0$.
We construct $q_s$ as follows.

For $s>0$, we define 
$$g_s(v,t) = \phi(v(te^{-\frac{1}{s}})) - v(te^{-\frac{1}{s}}).$$
Then we define 
$$q_s(x) = x+\int_0^1 g_s(v_s,t) dt,$$
where 
$$v_s(t) := x+\int_0^t g_s(v_s,z) dz.$$

%We aim to prove that $q_s$ converges compactly to $\phi$ as $s\to 0$.
%That is, for any compact set $K\subset\mathbb{R}$, 
% $$
% \lim_{s\to 0}\max_{x\in K}|q_s(x)-\phi(x)|=0.
% $$

We prove that for any $x$,
$q_s(x)$ converges to $\phi(x)$ as $s\to 0$.
In fact, we will show that the convergence rate is $O(e^{-\frac{1}{s}})$.

First we prove that $\max\limits_{t\in [0,1]}|v_s(t)|$ is uniformly bounded for $s > 0$.
To show this, we investigate the differential equation that $v_s(t)$ satisfies.
For notational convenience, we drop the subscript $s$ in $v_s$ in the proof below.
The dependence on $s$ will be stated explicitly when needed.
Note that 
$$\dfrac{dv}{dt} = g_s(v,t) = \phi(v(te^{-\frac{1}{s}})) - v(te^{-\frac{1}{s}}),\quad v(0)=x.$$
Equivalently, by a change of variable $\tau=te^{-\frac{1}{s}}$, we have
\begin{equation}
\label{eq:dvdtau}
\dfrac{dv}{d\tau} = e^{\frac{1}{s}} [ \phi(v(\tau)) - v(\tau)],\quad v(0)=x.
\end{equation}
In the following, we first analyze the property of the solution $v$ to the initial value problem in \eqref{eq:dvdtau},
and then we prove the uniform boundedness.

\textbf{Results on initial value problem (\ref{eq:dvdtau}).}
We prove in the following that the solution $v(\tau)$ of (\ref{eq:dvdtau}) must fall into one of the three cases below:
\begin{enumerate}[(I)]
\item $v'(\tau)=0$ for all $\tau$;    
\item $v'(\tau)>0$ for all $\tau$;    
\item $v'(\tau)<0$ for all $\tau$.
\end{enumerate}

Case (I): We first show that if $v'(a)=0$ for some $a\geq 0$, then  $v'(\tau)=0$ everywhere.
In fact, $v'(a)=e^{\frac{1}{s}}[\phi(v(a)) - v(a)]=0$ implies $\phi(v(a)) - v(a)=0$.
Note that $v(\tau)=v(a)$ is then an equilibrium solution. Since $\phi$ is Lipschitz, from the uniqueness theorem of the initial value problem, $v(\tau)=v(a)$ is the only solution and thus $v'(\tau)=0$ for all $\tau$. 

Case (II): If $v'(0)>0$, then it is easy to see that $v'(\tau)>0$ for all $\tau$.
In fact, if $v'(a)=0$ for some $a>0$, then we know from the result above that $v'(0)=0$, a contradiction; if $v'(a)<0$ for some $a>0$, then because $v'(\tau)$ is continuous, intermediate value theorem implies that there must be a point $b \in (0,a)$ such that $v'(b)=0$, which then implies $v'(0)=0$ according to the result above, a contradiction.

Case (III): If $v'(0)<0$, a similar argument shows that $v'(\tau)<0$ for all $\tau$.

Thus we conclude that there can only be three cases for the solution $v(\tau)$, as shown in (I), (II), (III).

\textbf{Proof of uniform boundedness of $|v_s|$.}
Next we show that every solution $v$ of (\ref{eq:dvdtau}) is uniformly bounded in $s$.

If $v$ falls into Case (I), it is easy to see that $v(\tau)=v(0)=x$ independent of $s$, thus uniformly bounded.

If $v$ falls into Case (II), then $\phi(v)-v>0$, and the equation in (\ref{eq:dvdtau}) can be equivalently written as
$$\dfrac{dv}{\phi(v) - v} = e^{\frac{1}{s}} d\tau.$$
Let $G(x)$ denote the anti-derivative of 
$\frac{1}{\phi(x) - x}$. Then it follows that
$$G(v)=\int_0^{\tau} e^{\frac{1}{s}} d\tau + A = \tau e^{\frac{1}{s}} + A = t+A,$$
where $A$ is a constant independent of $s$.
In fact, setting $t=0$ (or equivalently, $\tau=0$) yields that  
$$G(v(0)) = G(x) = 0+A = A.$$
Thus $A=G(x)$.
Since $G'(v) = \frac{1}{\phi(v) - v}= \frac{1}{v'} \cdot e^{\frac{1}{s}}  >0$,
we know from inverse function theorem that $G^{-1}$ exists and is continuous and strictly increasing.
Therefore,
\begin{equation}
    \label{eq:vcase2}
v = G^{-1}(G(v)) = G^{-1}(t+A) = G^{-1}(t+G(x))
\end{equation}
is uniformly bounded in $s$ and $t$ since $G$ is independent of $s,t$, and $t+G(x)\in [G(x),1+G(x)]$ with $t\in [0,1]$.
Therefore, in Case (II), $|v_s(t)|$ is uniformly bounded in $s$ and $t$.

If $v$ falls into Case (III), the uniform boundedness of $|v_s(t)|$ can be derived analogously as in Case (II).

%(I) $\phi(v(te^{-\frac{1}{s}})) - v(te^{-\frac{1}{s}}) > 1$;
%(II)  $\phi(v(te^{-\frac{1}{s}})) - v(te^{-\frac{1}{s}}) < -1$;
%(III)  $\phi(v(te^{-\frac{1}{s}})) - v(te^{-\frac{1}{s}}) \in [-1,1]$.

%An essentially same argument applied to Case (II) shows that $|v_s(t)|$ is uniformly bounded in $s$ and $t$ in that case.

Now we conclude that $\max\limits_{t\in [0,1]}|v_s(t)|$ is uniformly bounded with respect to $s$.

\textbf{Proof of convergence $q_s(x)\to \phi(x)$.}
Next we show that the uniform boundedness of $|v_s|$ implies the convergence
$q_s(x)\to \phi(x)$ as $s\to 0$.
Since $\phi$ is Lipschitz, we see that $|g_s(v_s,t)|$ is also uniformly bounded in $s$ and $t$.
Thus $M:=\sup\limits_{s>0,t\in [0,1]}|g_s(v_s,t)| < \infty$.
Note that
$$\phi(x) = x + \int_0^1 \phi(x) -x dt.$$
Let $L$ denote the Lipschitz constant of $\phi$. Then we deduce from the definition of $q_s$ and $v_s$ that 
\begin{equation*}
\begin{aligned}
|q_s(x)-\phi(x)|
&=\left| \int_0^1 \phi(v_s(te^{-\frac{1}{s}}))-\phi(x) + x-v_s(te^{-\frac{1}{s}}) dt \right|\\
&\leq (L+1)\int_0^1 |v_s(te^{-\frac{1}{s}})-x| dt \\
&= (L+1)\int_0^1 \left|\int_0^{te^{-\frac{1}{s}}} g_s(v_s,z) dz\right| dt\\
&\leq (L+1) \int_0^1 Mte^{-\frac{1}{s}} dt\\
&= (L+1)\frac{M}{2} e^{-\frac{1}{s}} \to 0,\quad s\to 0.
\end{aligned}
\end{equation*}
This proves the pointwise convergence $q_s(x)\to\phi(x)$ as $s\to 0$.
Since $K$ is compact and $|v_s|$ is continuous in $x$ (see \eqref{eq:vcase2} for example), it can be deduced from the argument above that $|v_s|$ is in fact uniformly bounded in $s>0$, $t\in [0,1]$ and $x\in K$.
Thus the convergence proof still holds with constant $M$ chosen as an upper bound of $|g_s|$ over $s>0$, $t\in [0,1]$ and $x\in K$.
The rate of convergence is still $O(e^{-\frac{1}{s}})$ as $s\to 0$.

The proof of Theorem 3 is now complete.

\end{proof}

% \convergence*
\begin{proof}[Proof of Theorem \ref{thm:convergence}]
This is an immediate result of Theorem \ref{thm:dense}.
Since compact convergence implies pointwise convergence, it suffices to show the compact convergence.
According to the proof of Theorem \ref{thm:dense},
for each $F_k$ (where each entry is a monotone continuous function), we can construct a family of triangular AUTM transformations $T_{s,k}$ (parametrized by $s>0$) that converge compactly to $F_k$ with a rate of $O(e^{-\frac{1}{s}})$ as $s\to 0$.
Then it follows immediately that $T_s:=T_{s,1}\circ T_{s,2}\circ \cdots\circ T_{s,p}$ converges compactly to $F=F_1\circ F_2\circ\cdots\circ F_p$ with rate $O(e^{-\frac{1}{s}})$ as $s\to 0$,
which completes the proof.
\end{proof}

% \family*
\begin{proof}[Proof of Theorem \ref{thm:family}]
The proof follows essentially the same argument as the proof of Theorem \ref{thm:dense} in which $\kappa_s(t)=1$.
More precisely, for a general positive kernel $\kappa_s(t)$ satisfying \eqref{eq:kernelConditions}, 
the positivity of $\kappa_s(t)$ is used in obtaining results for the initial value problem;
the bounded $L^1$ norm of $\kappa_s(t)$ is used in proving the uniform boundedness of $|v_s|$;
the asymptotic property as $s\to 0$ is used in proving the convergence $q_s\to \phi$.
Therefore, similar to the proof of Theorem \ref{thm:dense}, we conclude that $q_s|_K\to \phi|_K$ uniformly as $s\to 0$.
We remark that the class of kernels in \eqref{eq:kernelConditions} includes $\kappa_s(t)=1$, the normalized Gaussian kernel $\kappa_s(t)=C_s e^{-\frac{t^2}{s}}$ with $C_s=\left(\int_0^{1} e^{-\frac{z^2}{s}}dz\right)^{-1}$, and it can be computed that the convergence rate for the latter is $O(s^{-\frac{1}{2}}e^{-\frac{1}{s}})$.
\end{proof}

\section{Experiment details}
\label{sec:ExpAppendix}

\subsection{Hyperparameters for density estimation datasets}\label{appendix_experiment_details_data}

We list the hyperparameters in Table \ref{appendix_hyperparamter_5dataset}. Hyperparameters are obtained after extensive grid search. For the number of layers, we tried 5,10,20. For the hidden layer dimensions, we tried $10d,20d,40d$, where $d$ is the dimension of the vector in the dataset.  We trained our model by using Adam. We stop the training process when there is no improvement on validation set in several epochs.

\begin{table*}[ht]
  \small
  \centering
  \caption{Hyperparameters for Power, GAS, Hepmass, Miniboone, BSDS300 datasets, $d$ is the dimension of the vector in the dataset}
  \label{appendix_hyperparamter_5dataset}
  
  \vskip.05in
  \begin{tabular}{c|ccccc}
    \toprule
     & POWER & GAS & Hepmass & Miniboone & BSDS300\\
    \midrule
    layers & 10 &10 &10 &5 &10 \\
    hidden layer dimensions & 40d & 40d & 40d & 10d & 40d\\
    epochs & 450 & 1000 & 500 & 1000 & 1000\\
    batch size & 256 & 256 &256 & 256 & 128 \\
    optimizer & adam & adam& adam& adam& adam \\
    learning rate & 0.01 & 0.01& 0.01& 0.01& 0.01 \\
    lr decay rate & 0.5 & 0.5 & 0.5 & 0.5 & 0.5\\
    \bottomrule
  \end{tabular}
\end{table*}

\subsection{Hyperparameters for CIFAR10 and ImageNet32}\label{appendix_experiment_details}

We list the hyperparameters in Table \ref{appendix_hyperparamter_image}. In this experiment, most hyperparameters come from \cite{deepresiduallearning16}. We use 14 AUTM coupling layers with 8 residual blocks for each layer in our model. Like \cite{GLOW18}, before each coupling layers, there is an actnorm layer and a conv $1 \times 1$ layer. Each residual block has three convolution layers with 128 channels. Our method is trained for 100 epochs with batch size 64. We trained our model by using Adamax with Polyak.

\begin{table*}[ht]
  \small
  \centering
  \caption{Hyperparameters for CIFAR-10 and ImageNet32 datasets}
  \label{appendix_hyperparamter_image}
  
  \vskip.05in
  \begin{tabular}{c|cc}
    \toprule
     & CIFAR10 & ImageNet32\\
    \midrule
    layers & 14 &14 \\
    residential blocks & 8& 8 \\
    hidden channels & 128 &128\\
    epochs & 2500 & 50\\
    batch size & 64 &64 \\
    optimizer & adamax & adamax \\
    learning rate & 0.01 & 0.01 \\
    lr decay & 0.5 & 0.5\\
    lr decay epoch & [30,60,90] & [30,60,90] \\
    \bottomrule
  \end{tabular}
\end{table*}

\section{Reconstruction on Image dataset}

We examine the reserve step of our AUTM layer by showing the reconstruction of images. The used model is the same as the model in Section \ref{section_image_dataset} and use CIFAR10 and ImageNet32 dataset in this experiment. We compute the inverse of our layer by using iterative method with the reverse of integral as the initial guess. As Figure \ref{figure_reconstruction} shows, the average L1 reconstruction error converges in 15 steps. Also, Figure \ref{figure_reconstruction} shows that the reconstructed images look the same as original images.

We show the result of the reconstruction process of our method in Figure \ref{figure_reconstruction}.

\begin{figure*}[ht]  \centering  \includegraphics[width=.35\linewidth]{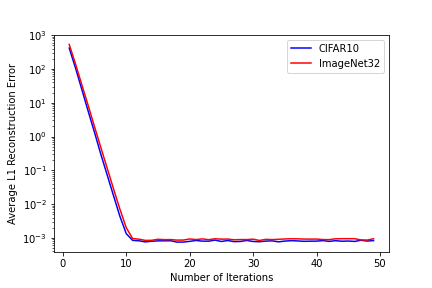}  \hspace{1cm}  
\includegraphics[width=.35\linewidth]{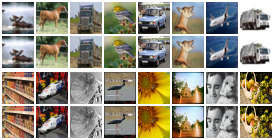}  
\caption{\textbf{Left}: The average value of the L1 of reconstruction error for 64 images. \textbf{Right}: The reconstruction of selected images in CIFAR10 and ImageNet32 dataset. The 1st, 3rd rows are the original images, and the 2nd, 4th rows are the reconstructions.
}  
\label{figure_reconstruction}
\end{figure*}

\section{Code}
Our code is available at {\url{https://anonymous.4open.science/r/AUTM-2B1B}.} We use some code from BNAF\cite{bnaf20}.

\end{document}